\documentclass[twoside]{aiml18}

\usepackage{aiml18macro}

\usepackage[utf8]{inputenc}
\usepackage{amsmath,amssymb,graphicx}
\usepackage{enumerate} 
\usepackage{url}
\usepackage{charter}
\usepackage{prooftrees}
\renewcommand{\phi}{\varphi}
\renewcommand{\epsilon}{\varepsilon}
\newcommand{\M}{\mathcal{M}}
\newcommand{\F}{\mathcal{F}}

\newcommand{\N}{\mathcal{N}}

\newcommand{\lra}{\leftrightarrow}

\newcommand{\K}{\ensuremath{\mathsf{K}}}

\newcommand{\hK}{\ensuremath{\widehat{\K}}}

\newcommand{\ELAS}{\ensuremath{\mathbf{ELAS}}}
\newcommand{\EL}{\ensuremath{\mathbf{EL}}}
\newcommand{\SELAS}{\ensuremath{\mathsf{SELAS}}}

\newcommand{\AxEX}{\ensuremath{\mathsf{EX}}}

\newcommand{\lr}[1]{\langle #1 \rangle}

\newcommand{\Nm}{\textbf{N}}

\newcommand{\Ps}{\textbf{P}}

\newcommand{\Var}{\textbf{X}}

\newcommand{\TAUT}{\ensuremath{\mathtt{TAUT}}}
\newcommand{\NECK}{\ensuremath{\mathtt{NECK}}}

\newcommand{\DISTK}{\ensuremath{\mathtt{DISTK}}}
\newcommand{\AxTrK}{\ensuremath{\mathtt{Tx}}}
\newcommand{\AxTrx}{\ensuremath{\mathtt{T}}}

\newcommand{\AxTransK}{\ensuremath{\mathtt{4x}}}
\newcommand{\AxEucK}{\ensuremath{\mathtt{5x}}}

\newcommand{\MP}{\ensuremath{\mathtt{MP}}}

 \newcommand{\AxEAS}{\ensuremath{\mathtt{EAS}}}
 
  \newcommand{\CNECAS}{\ensuremath{\mathtt{CNECAS}}}

\newcommand{\AxId}{\ensuremath{\mathtt{ID}}}
\newcommand{\AxSym}{\ensuremath{\mathtt{SYM}}}
\newcommand{\AxTranseq}{\ensuremath{\mathtt{TRANS}}}

\newcommand{\AxKT}{\ensuremath{\mathtt{KT}}}

\renewcommand{\iff}{\Leftrightarrow}

\renewcommand{\vec}{\overline}

\usepackage{color}
\usepackage[all]{xy}

\newcommand{\AxEXEAS}{\texttt{DAS}}
\newcommand{\AxDETAS}{\texttt{DETAS}}
\newcommand{\AxKAS}{\texttt{KAS}}
\newcommand{\AxEFAS}{\texttt{EFAS}}
\newcommand{\AxSUBAS}{\texttt{SUBAS}}
\newcommand{\AxSUBP}{\texttt{SUBP}}
\newcommand{\AxSUBK}{\texttt{SUBK}}
\newcommand{\AxRGDP}{\texttt{RIGIDP}}
\newcommand{\AxRGDN}{\texttt{RIGIDN}}
\newcommand{\AxSUBASEQ}{\ensuremath{\mathtt{SUBASEQ}}}
\newcommand{\AxDBASEQ}{\ensuremath{\mathtt{DBASEQ}}}
\newcommand{\NECAS}{\ensuremath{\mathtt{NECAS}}}
\newcommand{\AxSUBtoAS}{\texttt{SUB2AS}}
\newcommand{\Tr}{\textsf{Tr}}
\newcommand{\Fv}{\textsf{Fv}}
\newcommand{\Va}{\textsf{Var}}

\def\lastname{Wang, Seligman}

\begin{document}

\begin{frontmatter}
  \title{When Names Are Not Commonly Known: Epistemic Logic with Assignments}
 \author{Yanjing Wang}
  \address{Department of Philosophy, Peking University }
     \author{Jeremy Seligman}
  \address{Department of Philosophy, University of Auckland}

  \begin{abstract}
In standard epistemic logic, agent names are usually assumed to be common knowledge implicitly. This is unreasonable for various applications. Inspired by term modal logic and assignment operators in dynamic logic, we introduce a lightweight modal predicate logic where names can be non-rigid. The language can handle various \textit{de dicto} \slash \textit{de re} distinctions in a natural way. The main technical result is a complete axiomatisation of this logic over S5 models.  
\end{abstract}


 \end{frontmatter}

\section{Introduction}

One dark and stormy night, Adam was attacked and killed. His assailant, Bob, ran away, but was seen by a passer-by, Charles, who witnessed the crime from start to finish. This led quickly to Bob's arrest. Local news picked up the story, and that is how Dave heard it the next day, over breakfast. Now, in one sense we can say that both Charles and Dave know that Bob killed Adam. But there is a difference in what they know about just this fact. Although Charles witnessed the crime, and was able to identify the murderer and victim to the police, he might have no idea about their names. If asked ``Did Bob kill Adam?" he may not know. Yet this is a question that Dave could easily answer, despite not knowing who Adam and Bob are---he is very unlikely to be able to identify them in a line-up. 

The distinction between these \textit{de re} and \textit{de dicto} readings of ``knowing Bob killed Adam'' is hard to make in standard epistemic logic, where it is implicitly assumed that the names of agents are rigid designators and thus that it's common knowledge to whom they refer. But in many cases, the distinction is central to our understanding. On the internet, for example, users of websites and other online applications typically have multiple identities, and may even be anonymous. Names are rarely a matter of common knowledge and distinctions as to who knows who is whom are of great interest.  

Further complexities arise with higher-order knowledge and belief. In \cite{grove1995naming}, Grove gives an interesting example of a robot with a mechanical problem calling out for help (perhaps in a Matrix-like future with robots ruling the world unaided by humans). To plan further actions, the broken robot, called $a$, needs to know if its request has been heard by the maintenance robot, called $b$. But how to state exactly what  $a$ \textit{needs to know}? In English we would probably write it as: 
\begin{enumerate}
\item[($\star$)] $a$ knows that $b$ knows that $a$ needs help. 
\end{enumerate}
A naive formulation in standard (predicate) epistemic logic is $\K_a\K_b H(a)$. But without the assumption that the robots' names are both commonly known, there are various ambiguities. For example, if $b$ does not know which robot is named `$a$' then neither does $b$ know whom to help nor has $a$ any confidence of being helped. 
On the other hand, $a$ may not know that `$b$' is the name of the maintenance robot, thus merely knowing $b$ knows $a$ needs help is not enough for $a$ to be sure it will be helped.
The authors of \cite{corsi2013free} list several possible readings of ($\star$), which we will elaborate as follows: $a$, the broken robot, knows that
\begin{enumerate}[(i)]
\item the robot named `$b$' knows that the robot named `$a$' needs help, or 
\item the robot named `$b$' knows that it, i.e. the broken robot, needs help, or
\item the maintenance robot knows that the robot named `$a$' needs help, or 
\item the maintenance robot knows that it, i.e. the broken robot, needs help.
\end{enumerate}
It is impossible to distinguish the above readings in standard epistemic logic. In the literature \cite{grove1993naming,grove1995naming,corsi2013free,Fitting98,HollidayP2014}, various approaches are proposed.  In \cite{grove1995naming}, Grove correctly pinpoints the problems of \textit{scope} and \textit{manner of reference} in giving various \textit{de re} \slash \textit{de dicto} readings for higher-order knowledge, and proposes a new semantics for 2-sorted first-order modal logic that is based on world-agent pairs, so as to cope with indexicals like ``\texttt{me}''. A special predicate symbol `\texttt{In}' is introduced to capture scope explicitly: $\texttt{In}(a, b, n)$ holds at a world $w$ iff $b$ is someone named $n$ by $a$ in $w$. In \cite{Fitting98,thalmann2000term}, an intensional first-order modal logic uses predicate abstraction to capture different readings. $(\lambda x.\K_bHx)(a)$ says that agent $b$ knows \emph{de re} that $a$ is in need of help, whether or not $b$ knows that agent is named `$a$', whereas $\K_bH(a)$ says that $b$ knows \emph{de dicto} that someone called `$a$' needs help, whether or not $b$ knows who $a$ is. The authors of \cite{corsi2013free} propose a very general framework with complex operators based on counterpart semantics.\footnote{The counterpart semantics helps to handle the situation in which one agent is mistakenly considered as two people, as illustrated in \cite{corsi2013free} by the story of the double agent in Julian Symon’s novel \textit{The Man who killed himself.}} Without going into details, the formula $|t: \genfrac{}{}{0pt}{}{t_1\dots t_n}{x_1\dots x_n}|\phi(x_1\dots x_n)$ means, roughly, that the agent named by term $t$ knows \emph{de re} that $\phi$ of the things denoted by terms $t_1\dots t_n$. Holliday and Perry also bring the alethic modality into the picture together with the doxastic modality, and highlights the use of \textit{roles} to capture subtle readings in \cite{HollidayP2014}, where the multi-agent cases are handled by perspective switching based on a single-agent framework. 

In this paper, we follow the \textit{dynamic term modal logic} approach proposed by Kooi \cite{kooi2007dynamic}, based on  \textit{term modal logic} proposed in \cite{fitting2001term}. Term modal logic uses terms to index modalities which can also be quantified, so that $\K_{f(a)}\neg \forall x\K_x \phi$ says that $a$'s father knows that not everyone knows $\phi$. The accessibility relation used in the semantics of $\K_t$, where $t$ is a term, is then relative to the world $w$ at which this formula is evaluated: it is the one labeled by the agent denoted by term $t$ in $w$. Based on this, Kooi  \cite{kooi2007dynamic} borrows dynamic assignment modalities from (first-order) dynamic logic so as to adjust the denotation of names, now assumed to be non-rigid in general, in contrast to the usual \textit{constants} of first-order modal logic which are assumed rigid. 

Full first-order term modal language is clearly undecidable. In \cite{Padmanabha2018}, it is shown that even its propositional fragment is undecidable,\footnote{Only the monodic fragment is decidable \cite{Padmanabha2018}.} and the addition of the program modalities in dynamic logic makes things worse. As Kooi remarks in \cite{kooi2007dynamic}, the combination of term modal logic and dynamic assignment logic is not even recursively enumerable. A closely related study is the doctoral thesis of Thalmann \cite{thalmann2000term}, which provides many results including sequent calculi and tableaux systems for both term modal logic (with quantifiers) and quantifier-free dynamic assignment logic (with regular program constructions). But the two logics are studied \textit{separately}, leaving their combination as future work.\footnote{Thalmann predicts in his conclusion that ``\textit{Using the scoping operator instead of the quantifiers in term-modal logic, should lead to many interesting decidable fragments of term-modal logic.}''}  It is shown that the quantifier-free part of dynamic assignment logic is undecidable with both (Kleene) star operator and (rigid) function symbols but it is decidable if there is no star operator.\footnote{The later is only stated without a proof.} 
A rich treatment of various issues of `semantic competence' with names that uses term modal logic is given by Rendsvig in \cite{rendsvig2012}.

In this paper, we take a minimalist approach, introducing only the basic assignment modalities from dynamic logic combined with a quantifier-free term modal logic,  without function symbols, to obtain a small fragment of the logic in \cite{kooi2007dynamic}, which we conjecture to be decidable over S5 models (see discussions at the end of the paper). However, as we will soon see, it is already a very powerful tool for expressing various \textit{de re\slash de dicto} distinctions, as well as a kind of \textit{knowing who}, which was discussed by Hintikka \cite{Hintikka:kab} at the very inception of epistemic logic.\footnote{See \cite{WangBKT} for a summary of related works on knowing-wh.} The language is very simple and intuitive to use as a genuine multi-agent epistemic logic that does not presuppose common knowledge of names.


\medskip

Before the formal details, let us first illustrate the ideas. As in predicate epistemic logic more generally, the formula $\K_aPb$ says that $a$ knows \emph{de dicto} that $b$ is $P$, whereas  $\K_aPx$ says that $a$ knows \emph{de re} of $x$ that it is $P$.  The formula $[x:=b]Px$ says of $b$ that it is $P$, which is equivalent to $Pb$, but combining operators we get $[x:=b]\K_aPx$, which says that $a$ knows \emph{de re} of $b$ that it is $P$. More precisely, our semantics is based on first-order Kripke models with a constant domain of agents (not names) with formulas evaluated with respect to both a world $w$ and a variable assignment function $\sigma$.  Formula $[x:=t]\phi$ is then true iff $\phi$ is true at $w$ when we change $\sigma$ so that it assigns to $x$ the agent named by $t$ in $w$, and $\K_t\phi$ is true iff $\phi$ is true at all worlds indistinguishable from $w$ by the agent named $t$ in $w$. (This is in line with the \textit{innermost-scope} semantics of \cite{grove1995naming}.)


Returning to Grove's poor broken robot $a$, the various readings of `$a$ knows that $b$ knows that $a$ needs help' can be expressed as follows: 
\begin{enumerate}[(i)]
\item $\K_a\K_bH(a)$, $a$ knows that the robot named `$b$' knows that the robot named `$a$' needs help,
\item $[x:=a]\K_a \K_b H(x)$, $a$ knows that the robot named `$b$' knows it (the broken robot $a$) needs help,
\item $[y:=b]\K_a \K_y H(a)$, $a$ knows that it (the maintenance robot $b$) knows the robot named `$a$' needs help,
\item $[x:=a][y:=b]\K_a \K_y H(x)$, $a$ knows that it (the maintenance robot $b$) knows that it (the broken robot $a$) needs help.
\end{enumerate}
Moreover, since names are non-rigid, we can express \emph{$a$ knowing who $b$ is} by $[x:=b]\K_a (x\approx b)$ which says that $a$ identifies the right person with name $b$ on all relevant possible worlds. This we abbreviate as $\K_a b$.\footnote{There are a lot of different readings of \textit{knowing who}. E.g., knowing who went to the party may be formalized as $\forall x (\K W(x)\lor \K \neg W(x))$ under an exhaustive interpretation \cite{Wang17d}. See \cite{Aloni01,Aloni2018} for a very powerful treatment using \textit{conceptual covers} to give different interpretations. \cite{Rendsvigmaster} also contains related discussions.} Thus we are able to express the following:
\begin{enumerate}[(i)] \setcounter{enumi}{4}
\item $\neg \K_a a$: $a$ does not know he is called $a$ (c.f., ``the most foolish person may not know that he is the most foolish person" in \cite{kooi2007dynamic}). 
\item $b\approx c \land \K_ab \land \neg \K_ac$: $a$ knows who $b$ is but does not know who $c$ is, although they are just two names of the same person.  
\item $[x:=b][y:=a](\K_c M(x, y)\land\neg \K_c(a\approx x\land y\approx b))$: Charles knows who killed whom that night but does not know the names of the murderer and the victim.
\item $\K_dM(b, a)\land \neg \K_d a\land \neg \K_d b$: Dave knows that a person named Bob murdered a person named Adam without knowing who they are. 
\end{enumerate}

\medskip

The innocent look of our logical language belies some technical complexity. The main technical result is a complete axiomatisation of the logic over epistemic (S5) models (Section \ref{sec.ax} and \ref{sec.comp}), requires much work to handle the constant domain without Barcan-like formulas. We conclude with discussions on the issues of decidability of our logic (Section \ref{sec.con}).   

\section{Preliminaries}\label{sec.pre}
In this section we introduce formally the language and semantics of our logic. 
\begin{defn}[Epistemic language with assignments] Given a denumerable set of names $\Nm$, a denumerable set of variables \Var$, and $ a denumerable set $\Ps$ of predicate symbols, the language $\ELAS$ is defined as: 
$$ t::=  x \mid a 
$$
$$ \phi::= (t\approx t) \mid P\vec{t} \mid  \neg\phi \mid (\phi\land\phi)\mid \K_t \phi \mid [x:=t]\phi $$ 
\noindent where $a\in \Nm$, $P\in \Ps$, and  $\vec{t}$ is a vector of terms of length equal to the arity of predicate $P$. We write $\hK_t\phi$ as the the abbreviation of $\neg \K_t\neg \phi$ and write $\lr{x:=t}\phi$ as the abbreviation of $\neg [x:=t]\neg \phi$.\footnote{This is for comparison with other modal logics; in fact, the assignment modality is self-dual.} 
We call the $[x:=t]$-free fragment $\EL$.
\end{defn}
We define the semantics of $\ELAS$ over first-order Kripke models.  
\begin{defn}
A constant domain Kripke model $\M$ for \ELAS\ is a tuple $\lr{W, I, R, \rho, \eta}$ where: 
\begin{itemize}
\item $W$ is a non-empty set of possible worlds.
\item $I$ is a non-empty set of agents.
\item $R: I \to 2^{W\times W}$ assign a binary relation $R(i)$ (also written $R_i$) between worlds, to each agent $i$.
\item $\rho:\Ps\times W\to \bigcup_{n\in \omega}2^{I^n}$ assigns an $n$-ary relation $\rho(P,w)$ between agents to each $n$-ary predicate $P$ at each world $w$.
\item $\eta:\Nm\times W\to I$ assigns an agent $\eta(n,w)$ to each name $n$ at each world $w$.
\end{itemize}
We call $\M$ an \emph{epistemic model} if $R_i$ is an equivalence relation for each $i\in I$. 
\end{defn}

Note that the interpretations of predicates and names are not required to be rigid, and there may be worlds in which an agent has no name or multiple names. To interpret free variables, we need a variable assignment $\sigma: \Var\to I$. Formulas are interpreted on pointed models $\M,w$ with variable assignments $\sigma$. Given an assignment $\sigma$ and a world $w\in W$, let $\sigma_w(a)=\eta(a,w)$ and $\sigma_w(x)=\sigma(x)$. So although names may not be rigid, variables are.

The truth conditions are given w.r.t.\ pointed Kripke models with assignments $\M,w,\sigma$.
\begin{defn}
$$\begin{array}{|rcl|}
\hline
\M, w, \sigma\vDash t\approx t' &\Leftrightarrow & \sigma_w(t)=\sigma_w(t') \\ 
\M, w, \sigma\vDash P(t_1\cdots t_n) &\Leftrightarrow & (\sigma_w(t_1), \cdots, \sigma_w(t_n))\in \rho(P,w)  \\ 
\M, w, \sigma\vDash \neg\phi &\Leftrightarrow&   \M, w, \sigma\nvDash \phi \\ 
\M, w, \sigma\vDash (\phi\land \psi) &\Leftrightarrow&  \M, w, \sigma\vDash \phi \text{ and } \M, w, \sigma\vDash \psi \\ 
\M, w, \sigma\vDash \K_t \phi &\Leftrightarrow& \M, v, \sigma \vDash\phi \text{ for all $v$ s.t.\ $wR_{\sigma_w(t)}v$}\\
\M, w, \sigma\vDash [x:=t]\phi &\Leftrightarrow& \M, w, \sigma[x\mapsto \sigma_w(t)]\vDash \phi\\
\hline 
\end{array}$$
An $\ELAS$ formula is \textit{valid} (over epistemic models) if it holds on all the (epistemic) models with assignments $\M, s,\sigma$. 
\end{defn}

We can translate \label{standard-translation}
 $\ELAS$ into the corresponding (2-sorted) first-order language with not only the equality symbol but also a ternary relation symbol $R$ for the accessibility relation, a function symbol $f^a$ for each name $a$, and an $n+1$-ary relation symbol $Q^P$ for each predicate symbol $P$. The non-trivial clauses are for $\K_t$ and $[x:=t]\phi$ based on translation for terms: 
$$\begin{array}{l}
\Tr_w(x)=x \qquad \Tr_w(a)=f^a(w)\\
\Tr_w(t\approx t')= \Tr_w(t)\approx \Tr_w(t') \qquad \Tr_w(P\vec{t})= Q^P(w, \Tr_w(\vec{t})) \\
\Tr_w(\neg \psi)=\neg \Tr_w(\psi) \qquad \Tr_w(\phi\land\psi)=\Tr_w(\phi)\land \Tr_w(\psi).\\
\Tr_w(\K_t\psi)=\forall v (R(w,v,\Tr_w(t)) \to \Tr_{v}(\psi)) \\
\Tr_w([x:=t]\psi)=\left\{\begin{array}{l@{ \text{ if } }l}
\exists x (x\approx \Tr_w(t) \land \Tr_w(\psi))& t\not=x\\
\Tr_w(\psi)&t=x
\end{array}\right.
\end{array}
$$
Note that when $t\not= x$ we can also (equivalently) define $\Tr_w([x:=t]\psi)= 
\forall x (x\approx \Tr_w(t) \to \Tr_w(\psi))$, since there is one and only one value of $\Tr_w(t)$. When $x=t$, then $[x:=x]\psi$ is equivalent to $\psi$ according to the semantics.

In the light of this translation, we can define the free and bound occurrences of a variable in an $\ELAS$-formula by viewing $[x:=t]$  in $[x:=t]\phi$ as a quantifier for $x$ binding $\phi$. Note that the $t$ in $[x:=t]$ is not bound in $[x:=t]\phi$, even when $t=x$. The set of free variables $\Fv(\phi)$ in $\phi$ is defined as follows (where $\Va(\vec{t})$ is the set of variables in the terms $\vec{t}$):
\begin{center}
\begin{tabular}{ll}
     $\Fv(P\vec{t})=\Va(\vec{t})$&\\
   $\Fv(\neg \phi)=\Fv(\phi)$ & $\Fv(\phi\land \psi)=\Fv(\phi)\cup \Fv(\psi)$ \\
    $\Fv(\K_t \phi)=\Va(t)\cup \Fv(\phi)$ & $\Fv([x:=t]\phi)=(\Fv(\phi)\setminus \{x\})\cup \Va(t) $ \\
\end{tabular}
\end{center}
We use $\phi[y\slash x]$ to denote the result of substituting $y$ for all the free occurrences of $x$ in $\phi$.  We say $\phi[y\slash x]$ is \textit{admissible} if all the occurrences of $y$ by replacing free occurrences of $x$ in $\phi$ are also free.



We first show that \ELAS\ is indeed more expressive than \EL.
\begin{proposition}
The assignment operator $[x:=t]$ cannot be eliminated over (epistemic) models with variable assignments. 
\end{proposition}
\begin{proof}
Consider the following two (epistemic) models (reflexive arrows omitted) with a fixed domain $I=\{i, j\}$, worlds $W=\{s_1,s_2\}$ and a fixed assignment $\sigma(x)=i$ for all $x\in \Var$: 
\[
\begin{small}
\xymatrix{
\txt{$\M_1$:\\$\eta(a,s_1)=j$\\$\rho(P,s_1)=\emptyset$\\$s_1$ }\ar@{-}[d]^j & \txt{$\M_2$:\\ $\eta(a,s_1)=j$\\$\rho(P,s_1)=\emptyset$\\$s_1$}\ar@{-}[d]^j\\
\txt{$s_2$ \\$\rho(P,s_2)=\{i,j\}$\\$\eta(a,s_2)=i$} &\txt{$s_2$\\ $\rho(P,s_2)=\{i\}$\\$\eta(a,s_2)=i$} 
}
\end{small}
\]

$[x:=a]\hK_aPx$ can distinguish $\M_1, s_1$ and $\M_2, s_1$ given $\sigma$. But the only  atomic formulas other than identities are $Px$, $Pa$ and $a\approx x$, which are all false at $s_1$ and all true at $s_2$ in both models. Also note that $\K_a$ and $\K_x$ have exactly the same interpretation on the corresponding worlds in the two models. Based on these observations, a simple inductive proof on the structure of formulas would show that \EL\ cannot  distinguish the two models given $\sigma$. 
%
%
\end{proof}

Interested readers may also wonder whether we can eliminate $[x:=t]$ in each \ELAS\ formula to obtain an \EL\ formulas which is \textit{equally satisfiable}. However, the naive idea of translating $[x:=t]\phi$ into $z\approx t\land \phi[z\slash x]$ with fresh $z$ will not work in formulas like $\K_a[x:=c]\hK_b x\not\approx c$ since the name $c$ is not rigid. 

To better understand the semantics, the reader is invited to examine the following valid and invalid formulas over epistemic models:
\begin{center}
\begin{tabular}{|lrl|}
\hline
1 &valid & $x\approx y\to \K_t x\approx y$, \quad $ x\not \approx y \to \K_t  x\not\approx y$\\  
&invalid & $x\approx a\to \K_t x\approx a$, \quad $x\not \approx a \to \K_t x\not\approx a$, \quad $a\approx b\to \K_t a\approx b$ \\ 
2&valid & $\K_x\phi\to \K_x\K_x\phi$, \quad $\neg\K_x\phi \to \K_x\neg \K_x\phi$, \quad $\K_t\phi\to \phi$. \\
&invalid& $\K_t\phi\to \K_t\K_t\phi$, \quad $\neg\K_t\phi \to \K_t\neg \K_t\phi$ \\
3&valid &$[x:=y]\phi \to \phi [y\slash x]$ ($\phi [y\slash x]$ is admissible)\\ 
&invalid & $[x:=a] \K_tPx \to \K_tPa$ \\ 
4&valid &$x\approx a \to (\K_x \phi \to \K_a\phi)$, \quad $a\approx b\to (Pa\to Pb)$ \\ 
&invalid &$x\approx a \to (\K_b Px \to \K_a Pa)$, \quad $a\approx b\to (\K_cPa\to \K_cPb)$ \\ 
5&valid & $[x:=y]\K_a \phi \to  \K_a [x:=y] \phi $\\
&invalid& $[x:=b]\K_a Px \to  \K_a [x:=b] Px$\\ 
\hline
\end{tabular}
\end{center}

\begin{remark} Here are some brief explanations: 
\begin{itemize}
\item[1:] It shows the distinction between (rigid) variables and (non-rigid) names. The invalid formula shows that although two names co-refer, you may not know it (recall Frege's puzzle). 
\item[2:] Axioms \texttt{4} and \texttt{5} do not work for names in general, since $a$ may not know that he is named `$a$'. On the other hand, positive and negative introspection hold when the index is a variable. The $\mathsf{T}$ axiom works in general. 
\item[3:] It also demonstrates the non-rigidity of names. $[x:=a]\K_b Px$ does not imply $\K_bPa$ since $b$ may consider a world possible where $Pa$ does not hold since $a$ on that world does not refer to the actual person named by $a$ in the real world.
\item[4:] This shows that it is fine to do the first-level substitutions for the equal names but not in the scope of other modalities.
\item[5:] The last pair also demonstrates the distinction between rigid variables and non-rigid names. In particular, the analog of \textit{Barcan formula} $[x:=t]\K_s\phi \to \K_s[x:=t]\phi$ is not in general valid, if $t$ is a name. 
\end{itemize}
\end{remark}

\section{Axiomatisation} \label{sec.ax}
In this section we give a complete axiomatisation of valid \ELAS-formulas over epistemic models. The axioms and rules can be categorised into several classes:
\begin{itemize}
\item For normal propositional modal logic: $\TAUT, \DISTK, \MP, \NECK$; 
\item Axiom for epistemic conditions: $\AxTrK, \AxTransK, \AxEucK$;
\item Axioms for equality and first-level substitutability:  $\AxId,$ $\AxSUBP,$ $\AxSUBK,$ $\AxSUBAS$;
\item Axioms capturing rigidity of variables: $\AxRGDP$ and  $\AxRGDN$;
\item Properties of assignment operator: $\AxKAS$ (normality), $\AxDETAS$ (determinacy), $\AxEXEAS$ (executability), and $\AxEFAS$ (the effect of the assignment).
\item Quantifications: $\AxSUBtoAS$ and $\NECAS$, as in the usual first-order setting (viewing assignments as quantifiers). 
\end{itemize}
\begin{center}
\begin{tabular}{ll}
		\multicolumn{2}{c}{System $\SELAS$}\\
		{\textbf{Axioms}}&\\
	\begin{tabular}[t]{lll}
		\TAUT & \text{Propositional tautologies}\\
		\DISTK & $\K_t(\phi\to\psi)\to (\K_t\phi\to \K_t\psi)$\\
		\AxTrK& $\K_x \phi\to \phi $ \\
		\AxTransK& $\K_x \phi\to\K_x\K_x \phi$\\
		\AxEucK& $\neg \K_x \phi\to\K_x\neg\K_x \phi$\\
        \AxId & $t\approx t$\\
         \AxSUBP & $\vec{t}\approx \vec{t'} \to (P\vec{t}\lra P\vec{t'})$ \\ &($P$ can be $\approx$)\\
        \AxSUBK& $t\approx t' \to (\K_t \phi \lra \K_{t'} \phi)$\\
\end{tabular}&
\begin{tabular}[t]{lll}
       \AxSUBAS & $t\approx t' \to $ \\&$([x:=t] \phi \lra [x:=t']\phi)$\\
       \AxRGDP   & $x\approx y \to \K_t x\approx y$\\
   \AxRGDN        & $x\not\approx y \to \K_t  x\not\approx y$\\
\AxKAS&$[x:=t](\phi\to \psi)\to$\\& $([x:=t]\phi\to[x:=t]\psi)$ \\
\AxDETAS & $\lr{x:=t}\phi\to [x:=t]\phi$\\
\AxEXEAS & $\lr{x:=t}\top$\\
\AxEFAS  & $ [x:=t] x\approx t$ 
\\
\AxSUBtoAS & $\phi[y\slash x]\to [x:=y]\phi$ \\& ($\phi[y\slash x]$ is admissible)   \\
\end{tabular}
\end{tabular}
	\begin{tabular}{lclclc}
\textbf{Rules:}\\
 \MP & $\dfrac{\varphi,\varphi\to\psi}{\psi}$&\NECK& $\dfrac{\vdash\varphi}{\vdash\K_t\varphi}$ & \NECAS&$\dfrac{\vdash\varphi\to \psi}{\vdash \varphi \to [x:=t] \psi} \quad (x\not\in \Fv(\phi))$\\
 
\end{tabular}
\end{center}
\noindent where $\vec{t}\approx \vec{t'}$ means point-wise equivalence for sequences of terms $\vec{t}$ and $\vec{t'}$ such that $|\vec{t}|=|\vec{t'}|$.
It is straightforward to verify the soundness of the system. 
\begin{theorem}[Soundness]\label{th.soundness}
\SELAS\ is sound over epistemic models with assignments.  
\end{theorem} 
\begin{proposition}The following are derivable in the above proof system (where $\phi[y\slash x]$ is admissible below in $\AxSUBASEQ$) :
\begin{center}
\begin{tabular}{llll}
 \AxSym & $t\approx t'\to t' \approx t$ & \AxTranseq & $t \approx t'\land t'\approx t'' \to t\approx t''$\\
\AxDBASEQ & $\lr{x:=t}\phi\lra [x:=t]\phi$ & \AxSUBASEQ & $\phi[y\slash x]\lra [x:=y]\phi$ \quad \\
\AxEAS &$[x:=t] \phi \lra \phi$ ($x\not\in \Fv(\phi)$)& \AxTrx &$\K_t\phi\to \phi$ \\
\CNECAS &$\dfrac{\vdash\varphi\to \psi}{\vdash [x:=t] \varphi \to  \psi} \quad (x\not\in \Fv(\psi))$& \NECAS' &$\dfrac{\vdash\phi}{\vdash [x:=t] \varphi} $\\
\AxEX & $[x:=x]\phi\lra \phi$ & \\
\end{tabular}
\end{center}
\end{proposition}
\begin{proof}(Sketch)
 $\AxSym$ and $\AxTranseq$ are trivial based on $\AxId$ and $\AxSUBP$. $\AxDBASEQ$ is based on $\AxDETAS$ and $\AxEXEAS$. $\AxSUBASEQ$ is due to the contrapositive of $\AxSUBtoAS$ and $\AxDBASEQ$. $\CNECAS$ is due to $\NECAS$ and $\AxDBASEQ$ for contrapositive. $\AxEAS$ is based on $\NECAS$ and $\CNECAS$ (taking $\psi=\phi$). $\AxEX$ is a special case of $\AxSUBtoAS$ and $\NECAS'$ is a special case of $\NECAS$. As a more detailed example, let us look at the proof (sketch) of $\AxTrx$ (we omit the rountine steps using the normality of $[x:=t]$): 
$$
\begin{array}{lll}
(1) & \vdash\K_t\phi\to (z\approx t\to \K_z\phi) & \text{($\AxSUBK$, $z$ is fresh)} \\
(2) &  \vdash\K_t\phi\to[z:=t] (z\approx t\to \K_z\phi) &\text{($\NECAS (1)$)}\\ 
(3) & \vdash \K_t\phi\to[z:=t] (\K_z\phi) & \text{($\AxEFAS (2)$)}\\ 
(4) & \vdash [z:=t] (\K_z\phi\to \phi) & \text{($\NECAS', \AxTrK$)}\\ 
(5)& \vdash \K_t\phi\to[z:=t]\phi & \text{(normality of [z:=t] and $\MP$)}\\
(6)& \vdash \K_t\phi\to \phi & \text{($\AxEAS, \MP$)}
\end{array}
$$
\end{proof}


Based on the above result, we can reletter the bound variables in any \ELAS\ formula like in first-order logic. 
\begin{proposition}[Relettering]\label{prop.reletter}
Let $z$ be a fresh variable not in $\phi$ and $t$, then $$[x:=t]\phi\lra [z:=t]\phi[z\slash x]$$.
\end{proposition}
\begin{proof}
Since $z$ is fresh, $\phi[z\slash x]$ is admissible. We have the following proof (sketch): 
$$
\begin{array}{lll}
(1) & \vdash\phi[z\slash x]\lra [x:=z]\phi & \text{($\AxSUBASEQ$)} \\
(2) &  \vdash [z:=t]\phi[z\slash x]\lra [z:=t][x:=z]\phi &\text{(normality of $[z:=t]$)}\\ 
(3)&  \vdash [z:=t]\phi[z\slash x]\lra [z:=t](z\approx t\land [x:=z]\phi) &\text{($\AxEFAS$)}\\ 
(4)&  \vdash [z:=t]\phi[z\slash x]\lra [z:=t](z\approx t\land [x:=t]\phi) &\text{($\AxSUBAS$)}\\ 
(5)&  \vdash [z:=t]\phi[z\slash x]\lra [z:=t][x:=t]\phi &\text{($\AxEFAS$)}\\ 
(6)&  \vdash [z:=t]\phi[z\slash x]\lra [x:=t]\phi &\text{($\AxEAS$)}\\ 
\end{array}
$$
\end{proof}


\section{Completeness}\label{sec.comp}
To prove the completeness, besides the treatments of $[x:=t]$ and termed modality $\K_t$, the major difficulty is the lack of the Barcan-like formulas in \ELAS, which are often used to capture the condition of the constant domain. As in standard first-order logic, we need to provide witnesses for each name, and the Barcan formula can make sure we can always find one when building a successor of some maximal consistent set with enough witnesses. On the other hand, we can build an increasing domain pseudo model without such a formula using the techniques in \cite{Cresswell96}. Inspired by the techniques in \cite{Corsi02}, to obtain a constant domain model, when building the successors in the increasing domain pseudo model, we only create a new witness if all the old ones are not available, and we make sure by formulas in the maximal consistent sets that the new one is not equal to any old ones (throughout the whole model). In this way, there will not be any conflicts between the witnesses when we collect all of them together. We may then create a constant domain by considering the equivalence classes of all the witnesses occurring in the pseudo model with an increasing domain.

Here is the general proof strategy: 
\begin{itemize}
\item Extend the language with countably many new variables.
\item Build a pseudo canonical frame using maximal consistent sets for various sublanguages of the extended language, with witnesses for the names. 
\item Given a maximal consistent set, cut out its generated subframe from the pseudo frame, and build a constant-domain canonical model, by taking certain equivalence classes of variables as the domain. 
\item Show that the truth lemma holds for the canonical model. 
\item Take the reflexive symmetric transitive closure of the relations in pseudo model and show that the truth of the formulas in the original language are preserved. 
\item Extend each consistent set of the original model to a maximal consistent set with witnesses.  
\end{itemize}
\medskip

We first extend the language $\ELAS$ with countably infinitely many new variables, and call the new language $\ELAS^+$ with the variable set $\Var^+$. We say a language $L$ is an \textit{infinitely proper sublanguage} of another language $L'$ if:
\begin{itemize}
\item $L$ and $L'$ only differ in their sets of variables, 
\item $L\subseteq L'$,
\item there are infinitely many new variables in $L'$ that are not in $L$.
\end{itemize}
We use maximal consistent sets w.r.t.\ different infinitely proper sublanguages of $\ELAS^+$ that are extensions of $\ELAS$ to build a pseudo canonical frame.  

\begin{defn}[Pseudo canonical frame]
The pseudo canonical frame $\F^c=\lr{W, R}$ is defined as follows:
\begin{itemize}
\item $W$ is the set of MCS $\Delta$ w.r.t.\ some infinitely proper sublanguages $L_\Delta$ of $\ELAS^+$ such that for each $\Delta\in W$:
\begin{itemize}
\item $\ELAS\subseteq L_\Delta$,
\item For each $a\in\Nm$ there is a variable $x$ in $L_\Delta$ (notation: $x\in \Va(\Delta))$ such that $x\approx a \in \Delta$ (call it $\exists$-property)
\end{itemize}
\item For each $x\in \Var^+$, $\Delta R_{x} \Theta$ iff the following three conditions hold:
\begin{enumerate}
\item $x$ in  $\Va(\Delta)$, the set of variables in $L_\Delta$.
\item $\{\phi \mid \K_x\phi\in \Delta\}\subseteq \Theta$.
\item if $y\in \Va(\Theta)\setminus \Va(\Delta)$ then $ y\not\approx z\in \Theta$ for all $z\in \Va(\Theta)$ such that $z\not=y$. 
\end{enumerate}
\end{itemize}
\end{defn}


\paragraph{Observation} The last condition for $R_x$ makes sure that every new variable in the successor is distinguished from any other variables by inequalities. It is also easy to see that if $t\in L_\Delta$ then there is $x\in \Va(\Delta)$ such that $x\approx t\in \Delta$ by $\exists$-property and $\AxId$.

\begin{proposition}\label{prop.obs}
If $\Delta R_x \Theta$ in $\F^c$, then:
\begin{itemize}
\item $L_\Delta$ is a sublanguage of $L_\Theta$
\item for any $y\not=z\in\Va(\ELAS^+)$:  $y\approx z \in \Delta$ iff $y\approx z\in \Theta$.
\end{itemize}
\end{proposition}
\begin{proof}
For the first: For all $y\in \Va(\Delta)$, $y\approx y\in\Delta$ therefore $\K_x(y\approx y)\in \Delta$ by $\AxRGDP$, thus $y\approx y\in \Theta$. 

For the second: 
\begin{itemize}
\item Suppose $y,z\in \Va(\Delta)$
\begin{itemize}
\item If $y\approx z \in \Delta$, then $\K_x y\approx z \in \Delta$ by $\AxRGDP$, thus $y\approx z\in\Theta.$
\item If $y\approx z\not\in\Delta$ then $y\not \approx z \in \Delta$ since $\Delta$ is an $L_\Delta$-MCS and $y,z\in\Va(\Delta)$. Then $\K_x y\not \approx z \in \Delta$ by $\AxRGDN$, thus $y\approx z\not\in\Theta$.
\end{itemize}
\item Suppose w.o.l.g. $y\not\in \Va(\Delta)$ thus $y\approx z\not\in \Delta$.
\begin{itemize}
\item If $y\not\in \Va(\Theta)$ or $z\not\in \Va(\Theta)$ then then $y\approx z \in \Delta$ iff $ y\approx z \in \Theta$ trivially holds. 
\item If $y\in \Va(\Theta)$ and $z\in \Va(\Theta)$ then 
$y\not\approx z\in\Theta$ due to the third condition of $R_x$. Therefore $y\approx z\not\in\Theta$ since $\Theta$ is consistent. Thus $y\approx z \in \Delta$ iff $y \approx z\in \Theta. $ 
\end{itemize}
\end{itemize}
\end{proof}
The second part of the above proposition makes sure that we do not have conflicting equalities in different states which are accessible from one to another. 

\begin{lemma}[Existence lemma] \label{lem.ex}
If $\Delta\in W$ and $\hK_t \phi \in \Delta$ then there is a $\Theta\in W$ and an $x\in \Va(L_\Delta)$ such that $\phi\in\Theta$, $x\approx t\in\Delta$, and $\Delta R_{x}\Theta$.
\end{lemma}
\begin{proof}
If $\hK_t \phi \in \Delta$  then there is $x\approx t \in\Delta$ for some $x$, due to the fact that $\Delta$ has the $\exists$-property. Let $\Theta^{--}=\{\psi\mid \K_x\psi \in \Delta\}\cup \{\phi\}$. We first show that $\Theta^{--}$ is consistent by $\DISTK$ and $\NECK$ (routine). Next we show that it can be extended to a state in $W$. We can select an infinitely proper sublanguage $L$ of $\ELAS^+$ such that $L_\Delta$ is an infinitely proper sublanguage of $L$. We can list the new variables in $L$ but not in $L_\Delta$ by $y_0, y_1, y_2, \dots$.  We also list the names in $\Nm$ as $a_0, a_1, \dots$. In the following, we add the witness to the names by building $\Theta_i$ as follows:

\begin{itemize}
\item $\Theta_0=\Theta^{--}$
\item $\Theta_{k+1}=\left\{\begin{array}{lll}
\Theta_k & \text{if}  &\text{ $x\approx a_k$ is in $\Theta_k$ for some $x\in \Va(\Delta)$ (1)}\\
\Theta_k\cup\{x_i\approx a_k\} & \text{if} &\text{(1) does not hold but $\{x\approx a_k\}\cup\Theta_{k}$}\\
& &\text{ is consistent for some $x\in\Va(\Delta)$,} \\
&& \text{ and $x_i$ is the first such $x$ according to} \\ &&\text{a fixed enumeration of $\Va(\Delta)$ \qquad(2) }\\
\Theta_k\cup\{y_j\approx a_k\}\cup  & \text{if} & \text{ neither (1) nor (2) holds and } \\ 
\{y_j\not \approx z\mid z \in \Va(\Theta_k)\} &&\text{$y_j$ is the  first in the enumeration }\\ && \text{ of the new variables not in $\Theta_k$}\qquad (3)
\end{array}
\right.
$
\end{itemize}
We can show that $\Theta_k$ is always consistent. Note that $\Theta_0$ is consistent, we just need to show if $\Theta_k$ is consistent and $(1), (2)$ do not hold, then $\Theta_{k+1}$ is consistent too.  Suppose for contradiction that $\Theta_k\cup\{y_j\approx a_k\}\cup \{y_j\not \approx z\mid z \in \Va(\Theta_k)\}$ is not consistent then there are fomulas $\psi_1\dots \psi_n\in\Theta_k$, and $z_{i_1}\dots z_{i_m}\in \Va(\Theta_k)$ such that: 
$$
\vdash \psi_1\land \dots \land \psi_n \land y_j\approx a_k \to \bigvee_{i \in \{i_1, \dots, i_m\}} y_j \approx z_i   \quad (\star)
$$
First note that since $y_j$ is not in $\Theta_{k}$, $\Theta_k\cup\{y_j\approx a_k\}$ is consistent, for otherwise there are $\psi_1\dots \psi_n\in\Theta_k$ such that $\vdash \bigwedge_{i\leq n} \psi_i\to y_j\not\approx a_k$, then by \NECAS, we have $\vdash  \bigwedge_{i\leq n} \psi_i\to [y_j:=a_k] y_j\not\approx a_k$ thus by $\AxEFAS$ we have $\vdash  \bigwedge_{i\leq n} \psi_i\to [y_j:=a_k](a_k\approx y_j\land a_k\not\approx y_j)$, contradicting to the consistency of $\Theta_k$ (by $\AxDETAS$). By $(\star)$, $\Theta_k\cup\{y_j\approx a_k\}$  is consistent with one of $y_j\approx z_i$ for some $z_i$ in $\Va(\Theta_{k})$. Thus $\Theta_k\cup\{z_i\approx a_k\}$ is also consistent which contradicts the assumption that condition (2) and (1) do not hold. 

Then we define $\Theta^-$ to be the union of all $\Theta_k$. Clearly, $\Theta^-$ has the $\exists$-property. We build the language $L'$ based on  $\Va(\Theta^-)$. Note that $L'$ is still an infinitely proper sublanguage of $\ELAS^+$. 

Finally, we extend $\Theta^-$ into an MCS w.r.t.\ $L'$ and it is not hard to show $\Delta R_x \Theta$ by verifying the third condition: when we introduce a new variable we always make sure it is differentiated with the previous one in the construction of $\Theta$. 
\end{proof}



Given a state $\Gamma$ in $\F^c$, we can define an equivalence relation $\sim_\Gamma$: $x\sim_\Gamma y$ iff $x\approx y \in \Gamma$ or $x=y$ (note that $x\approx x$ is \textit{not} in $\Gamma$ if $x\not\in L_\Gamma$). Due to $\AxId, \AxSym, \AxTranseq$, $\sim_\Gamma$  is indeed an equivalence relation. When $\Gamma$ is fixed, we write $|x|$ for the equivalence class of $x$ w.r.t.\ $\sim_\Gamma$. By definition, for all $x\not\in\Va(\Gamma)$, $|x|$ is a singleton. 

Now we are ready to build the canonical model. 

\begin{defn}[Canonical model]
Given a $\Gamma$ in $\F^c$ we define the canonical model  $\M_\Gamma=\lr{W_\Gamma, I^c, R^c,\rho^c,\eta^c}$  based on the psuedo canonical frame $\lr{W, R}$
\begin{itemize}
\item $W_\Gamma$ is the subset of $W$ generated from $\Gamma$ w.r.t. the relations $R_x$.
\item $I^c=\{ |x|\mid x\in\Va(W_\Gamma) \}$ where $\Va(W_{\Gamma})$ is the set of all the variables appearing in $W_{\Gamma}$.
\item  $\Delta R^c_{|x|}\Theta$ iff $\Delta R_x\Theta$, for any $\Delta, \Theta \in W_\Gamma$.
\item $\eta^c(a, \Delta)=|x|$ iff $a\approx x \in \Delta$.
\item $\rho^c(P, \Delta)=\{\vec{|x|}\mid P\vec{x}\in \Delta\}.$
\end{itemize}
\end{defn}

Here is a handy observation. 
\begin{proposition} \label{prop.newv}If $y \in \Va(\Delta)\setminus \Va(\Gamma)$ then $y\not\approx z\in \Delta$ for all $z\in \Va(\Delta)$ such that $z\not=y$.
\end{proposition}
\begin{proof}
Due to the condition 3 of the relation $R_x$ in $\F^c$ and Proposition \ref{prop.obs}, and the fact that $W_\Gamma$ is generated from $\Gamma$.  
\end{proof}

\begin{proposition}
The canonical model is well-defined. 
\end{proposition}
\begin{proof}
\begin{itemize}
\item For $R_{|x|}$: We show that the choice of the representative in $|x|$ does not change the definition. Suppose $x\sim_\Gamma y$ then either $x=y$ or $x\approx y\in\Gamma$. In the first case, $\Delta R_{x}\Theta $ iff $\Delta R_{y}\Theta $. In the second case, suppose $\Delta R_{x}\Theta $. We show that the three conditions for 
$\Delta R_{y}\Theta $ hold.  For condition 1, $y\in\Va(\Delta)$ since $y\in\Va(\Gamma)$ and $\Delta$ is generated from $\Gamma$ by $R$.  For condition 2, we just need to note that  $\vdash y\approx x\to (\K_{x}\phi \lra \K_{y}\phi)$ by $\AxSUBK$. And condition 3 is given directly by condition 3 for $\Delta R_{x}\Theta$.
\item For $\eta(a, \Delta)$: We first show that the choice of the representative in $|x|$ does not change the definition by $\vdash x\approx y \to (a\approx x\lra a\approx y)$ ($\AxTranseq$). Then we need to show that $\eta^c(a,\Delta)$ is unique. Note that due to the $\exists$-property, there is always some $x$ such that $x \approx a$ in $\Delta$. Suppose towards contradiction that $a\approx x\in \Delta$, $a\approx y\in \Delta$ and $x\not \sim_\Gamma y$ then clearly $x, y$ cannot be both in $\Va(\Gamma)$ for otherwise $x\not\approx y\in\Delta$. Suppose w.l.o.g.\ $x$ is not in  $\Va(\Gamma)$ then we should have $x\not\approx y\in \Delta$ due to Proposition \ref{prop.newv}, contradicting the assumption that $\Delta$ is consistent.  
\end{itemize}
\end{proof}
\begin{proposition}\label{prop.trans}
$R_{|x|}$ is transitive. 
\end{proposition}
\begin{proof}
Suppose $\Delta R_{|x|}\Theta$ and $\Theta R_{|x|} \Lambda$ then in $\F^c$ $\Delta R_x\Theta$ and $\Theta R_x\Lambda$ (note that the representative of $|x|$ does not really matter since $R_{|x|}$ is well-defined). We have to show the three conditions for $\Delta R_x\Lambda$.  For condition 1, $x\in\Va(\Delta)$ since $\Delta R_x\Theta$. For condition 2, by Axiom $\AxTransK$, we have for any $\phi$ such that $\K_x\phi\in \Delta$ we have $\K_x\K_x\phi\in\Delta$ thus $\K_x\phi\in \Theta$ thus $\phi\in \Lambda$, by the definition of $R_{x}$. For condition 3, suppose $y\in\Va(\Lambda)\setminus\Va(\Delta)$.  Then since $\Delta\in W$, $y\not\in\Va(\Gamma)$, so by Proposition~\ref{prop.newv} we are done.
\end{proof}

Before proving the truth lemma, we have two simple observations: 
\begin{proposition}. \label{prop.equiv}
\begin{itemize}
\item[(1)] If $x\approx y$ is in some $\Delta\in W_\Gamma$ then $x\sim_\Gamma y$.
\item[(2)] If $x\sim_\Gamma y$ then $x=y$ or $x\approx y$ in all the $\Delta\in W_\Gamma$. 
\end{itemize}
%
\end{proposition}
\begin{proof}
For the first, suppose $x\approx y$ is in some $\Delta\in W_\Gamma$. We just need to consider the case when $x\not=y$ for if $x=y$ then $x\sim_\Gamma y$ by definition. By Proposition \ref{prop.newv}, $x$ and $y$ must be both in $\Va(\Gamma)$, thus by $\AxRGDN$ and the fact that $\Delta$ is connected to $\Gamma$, $x\approx y\in \Gamma$.  

The second is immediate by the definition of $\sim_\Gamma$: if $x\approx y\in\Gamma$ then $x\approx y\in \Delta$ due to $\AxRGDP$ and the fact that all the $\Delta\in W_\Gamma$ are connected to $\Gamma$. 

\end{proof}

Although $R_{|x|}$ is transitive, the model $\M_\Gamma$ is not reflexive nor symmetric in general. For the failure of reflexivity, note that some $x$ may  not be in the language of some state. For the failure of symmetry: We may have $\Delta R_{|x|}\Theta$ and $L_\Delta\subset L_\Theta$ thus it is not the case that $\Theta R_{|x|}\Delta$ by Proposition \ref{prop.obs}. We will turn this model into an S5 model later on. Before that we first prove a (conditional) truth lemma w.r.t.\ $\M_\Gamma$ and the canonical assignment $\sigma^*$ such that $\sigma^*(x)=|x|$ for all $x\in \Va(W_\Gamma)$.
\begin{lemma}[Truth lemma] \label{lemtruth}For any $\phi\in \ELAS^+$ and any $\Delta \in W$, if $\phi\in L_\Delta$ then:  
$$\M_\Gamma, \Delta, \sigma^* \vDash \phi\iff \phi\in \Delta$$
\end{lemma}
\begin{proof}
We do induction on the structure of the formulas. 

For the case of $t\approx t'\in L_\Delta,$ by $\exists$-property we have some $x,y\in \Va(L_\Delta)$ such that $t\approx x\in \Delta, t'\approx y\in\Delta$. 
\begin{itemize}
\item Suppose $t\approx t'\in \Delta$. Since $t\approx x\in \Delta, t'\approx y\in\Delta$, by $\AxTranseq$, $x\approx y\in \Delta$. Now by Proposition \ref{prop.equiv}, $x\sim_\Gamma y$. Thus $\sigma^*(t,\Delta)=|x|=|y|=\sigma^*(t',\Delta)$, then $\M_\Gamma, \Delta, \sigma^* \vDash t\approx t'.$

\item If $\M_\Gamma, \Delta, \sigma^* \vDash t\approx t'$ then $ \sigma^*(t,\Delta)=\sigma^*(t',\Delta)$. If $t$ and $t'$ are variables $x,y$, then $|x|=|y|$ i.e., $x\sim_\Gamma y$. By Proposition \ref{prop.equiv}, either $x=y$ or $x\approx y\in \Delta$. Actually, even if $x=y$, since $x$ is in $\Va(\Delta)$, $x\approx x\in \Delta$ by $\AxId$. If $t$ and $t'$ are both in $\Nm$, then by the definition of $\eta^c$, there are $t\approx x$ and $t\approx y$ in $\Delta$ and $y\in |x|$, which means $x \sim_\Gamma y$. By Proposition \ref{prop.equiv} and $\AxId$ again, $x\approx y\in \Delta$ therefore $t\approx t'\in \Delta$. Finally, w.l.o.g.\ if $t\in\Va(\Delta)$ and $t'\in\Nm$, then by definition of $\eta$, $t'\approx x\in\Delta$ for some $x$ and $t\sim_\Gamma x$. Again, since $x, t\in \Va(\Delta)$, $x\approx t\in \Delta$ therefore $t\approx t'\in \Delta$.    
\end{itemize}

For the case of $P\vec{t}\in L_\Delta$. 
\begin{itemize}
\item If $P\vec{t}\in\Delta$, then by $\exists$-property, there are $\vec{x}$ in $\Va(\Delta)$ such that $\vec{x}\approx\vec{t}\in\Delta$. Then by $\AxSUBP$ we have $P\vec{x}\in\Delta$. Thus by the definition of $\rho^c$, $\vec{|x|}\in \rho^c(P, \Delta)$. By the definitions of $\sigma^*$ and $\eta^c$, $\vec{\sigma^*(t,\Delta)}=\vec{|x|}$. Therefore $\M_\Gamma, \Delta, \sigma^* \vDash P\vec{t}$. 
\item If $\M_\Gamma, \Delta, \sigma^* \vDash P\vec{t}$, then the vector $\vec{\sigma^*(t,\Delta)}\in \rho^c(P, \Delta)$. It means that $\vec{\sigma^*(t,\Delta)}$=$\vec{|x|}$ (coordinate-wise) for some $P\vec{x}\in\Delta$ such that $\vec{t}\approx \vec{y}\in\Delta$ for some $\vec{y}$ such that $\vec{x}\sim_\Gamma \vec{y}$. Note that since $P\vec{x}\in\Delta$, $\vec{x}\in\Va(\Delta)$. It is not hard to show that $\vec{x}\approx\vec{y}\in\Delta$ by Proposition \ref{prop.equiv}. Now based on $\AxSUBP$, $P\vec{t}\in\Delta$. 
\end{itemize}

The boolean cases are routine.

For the case of $\K_t\psi\in L_\Delta$: 
\begin{itemize}
\item Suppose $\K_t\psi\not\in \Delta$, then $\hK_t\neg \psi\in \Delta$. By Lemma \ref{lem.ex} there is some variable $x$ and $\Theta\in W_\Gamma$ such that $\Delta R_{|x|}\Theta$, $x\approx t\in \Delta$ and $\neg\psi\in\Theta$. Therefore, by the induction hypothesis, $\M_\Gamma, \Theta, \sigma^* \nvDash \psi$ and so $\M_\Gamma, \Delta, \sigma^* \nvDash \K_x\psi$. If $t$ is a variable then $x\sim_\Gamma t$ by Proposition\ \ref{prop.equiv}, thus $\sigma^*(\Delta, t)=|x|$. If $t$ is a name then by definition $\eta^c(\Delta, t)=|x|$. Therefore in either case we have $\M_\Gamma, \Delta, \sigma^* \nvDash \K_t\psi$.
\item Suppose $\K_t\psi\in \Delta$, then by $\exists$-property, there is an $x\in \Va(\Delta)$ such that $x\approx t\in \Delta$, thus $\K_x\psi\in\Delta$ by $\AxSUBK$ and $\sigma^*(t, \Delta)=|x|$. By induction hypothesis, $\M_\Gamma, \Delta, \sigma^*\vDash x\approx t$. Now consider any $R_{|x|}$-successor $\Theta$ of $\Delta$, it is clear that $\psi\in \Theta$ by definition of $R_{|x|}$. Now by induction hypothesis again, $\M_\Gamma, \Theta, \sigma^*\vDash \psi$. Therefore, $\M_\Gamma, \Delta, \sigma^* \vDash \K_t\psi$.
\end{itemize}

For the case of $[x:=t]\psi\in L_\Delta$: 
\begin{itemize}
\item Suppose $\M_\Gamma, \Delta, \sigma^* \vDash [x:=t]\psi$. 
\begin{itemize}
\item If $t\in \Nm$, by $\exists$-property, we have $y\approx t\in \Delta$ for some $y\in \Va(\Delta).$ By induction hypothesis, $\M_\Gamma, \Delta, \sigma^* \vDash y\approx t$. Therefore $\sigma^*(\Delta, t)=|y|$ thus  $\M_\Gamma, \Delta, \sigma^*[x\mapsto |y|] \vDash \psi$. Now if $\psi[y\slash x]$ is admissible then we have $\M_\Gamma, \Delta, \sigma^* \vDash \psi[y\slash x]$. By IH, $\psi[y\slash x]\in \Delta$. Thus $[x:=y]\psi\in\Delta$ by $\AxSUBtoAS$. Since $t\approx y\in \Delta$, thus $[x:=t]\psi\in\Delta$ by $\AxSUBAS$. Note that if $\psi[y\slash x]$ is not admissible, then we can reletter $\psi$ to have an equivalent formula $\psi'\in L(\Delta)$ such that $\psi'[y\slash x]$ is admissible. Then the above proof still works to show that $[x:=t]\psi'\in\Delta$. Since relettering can be done in the proof system by Proposition \ref{prop.reletter}, we have $[x:=t]\psi\in\Delta$.  
\item If $t$ is a variable $y$, then $\M_\Gamma, \Delta, \sigma^*[x\mapsto |y|] \vDash \psi$. From here a similar (but easier) proof like the above suffices.
\end{itemize}
\item Supposing $[x:=t]\psi\in\Delta$, by the $\exists$-property of $\Delta$, we have some $y\in \Va(\Delta)$ such that $t\approx y\in\Delta$. Like the proof above we can assume w.l.o.g. that $\psi[y\slash x]$ is admissible, for otherwise we can reletter $\psi$ first. Thus $[x:=y]\psi\in\Delta$ by $\AxSUBAS$. Then by $\AxSUBASEQ$, $\psi[y\slash x]\in \Delta$. By IH, $\M_\Gamma, \Delta, \sigma^*\vDash \psi[y\slash x] \land t\approx y$. By the semantics and the assumption that $\psi[y\slash x]$ admissible, $\M_\Gamma, \Delta, \sigma^* \vDash [x:=t]\psi$.
\end{itemize}
\end{proof}

\medskip

Now we will transform the canonical model into a proper S5 model by taking the reflexive, symmetric and transitive closure of each $R_{|x|}$ in $\M_\Gamma$. Note that although $\M_\Gamma$ is a transitive model, the symmetric closure will break the transitivity. Actually, it can be done in one go by taking the reflexive transitive closure via undirected paths. More precisely, let $\N_\Gamma$ be the model like $\M_\Gamma$ but with the revised relation $R^*_{|x|}$ for each $x\in\Va(W_\Gamma)$, defined as: 

\begin{tabular}{lll}
$\Delta R^*_{|x|} \Theta$  & $\iff$ & either $\Delta=\Theta$ or there are some $\Delta_1 \dots \Delta_n$ for some $n\geq 0$ \\
&&\text{ such that }$\Delta_k R_{|x|} \Delta_{k+1}$  or  $\Delta_{k+1} R_{|x|} \Delta_{k}$  \\ &&\text{for each $0\leq k\leq n$} where $\Delta_0=\Delta$ and $\Delta_{n+1}=\Theta$.
\end{tabular}

We will show that it preserves the truth value of $\ELAS$ formulas. 

\begin{lemma}[Preservation lemma] \label{lem.pre}
For all $\phi\in \ELAS:$
$$\N_\Gamma, \Delta, \sigma^* \vDash \phi\iff \phi\in \Delta$$
\end{lemma}
\begin{proof}
Since we only altered the relations, We just need to check $\K_t\psi\in\ELAS$. Note that then $\K_t\psi$ is in all the local language $L_\Delta$.
\begin{itemize}
\item If $\N^\Gamma, \Delta, \sigma^* \vDash \K_t\psi$ then since the closure only adds relations then we know $\M_\Gamma, \Delta, \sigma^* \vDash \K_t\psi$ by induction hypothesis and Lemma \ref{lemtruth}. Now by Lemma \ref{lemtruth} again $\K_t\psi\in \Delta.$ 
\item Suppose $\K_t\psi\in \Delta$. Since $\Delta$ has $\exists$-property, there is some $x\in\Va(\Delta)$ such that $x\approx t\in \Delta$ thus $\K_x\psi\in\Delta$. Now consider an arbitrary $R^*_{|x|}$-successor $\Theta$ in $\N^\Gamma$. If $\Delta=\Theta$ then by $\AxKT$ it is trivial to show that $\psi\in\Delta$. Now by the definition of $R^*_{|x|}$, suppose there are some $\Delta_1 \dots \Delta_n$ \text{ such that }$\Delta_k R_{|x|} \Delta_{k+1}$ or $\Delta_{k+1} R_{|x|} \Delta_{k}$ \text{ for each $0\leq k\leq n$ } where $\Delta=\Delta_0$ and $\Theta=\Delta_{n+1}$. Now we do induction on $n$ to show that $\K_x\psi\in \Delta_{k}$ for all those $k\leq n+1$. Note that if the claim is correct then by $\AxKT$ we have $\psi\in \Delta_{k+1}$ thus by 
IH we have $\N^\Gamma, \Delta, \sigma^* \vDash \K_t\psi$.
\begin{itemize}
\item $n=0:$ Then there are two cases: 
\begin{itemize}
\item $\Delta R_{|x|}\Theta$ in $\M_\Gamma$: by $\AxTransK$, $\K_x\K_x\psi\in \Delta$ and then $\K_x\psi\in\Theta$ by the definition of $R_{|x|}$. 
\item $\Theta R_{|x|} \Delta$ in $\M_\Gamma$: First note that there is some $y\in |x|$ such that $y\in\Va(\Theta)$ by the definition of $R_{|x|}$. If $y\not=x$ then by Proposition \ref{prop.equiv}, we have $y\approx x\in \Theta$, therefore  $x\in\Va(\Theta)$. Towards contradiction suppose $ \neg \K_x\psi\in\Theta$. By $\AxEucK$, $\K_x\neg \K_x\psi \in \Theta$. By definition of $R_{|x|}$, $\neg \K_x\psi\in \Delta$. Contradiction.
\end{itemize}
\item $n=k+1:$ Supposing that the claim holds for $n=k$, i.e., $\K_x\psi\in \Delta_{k}$. There are again two cases: $\Delta_k R_{|x|} \Delta_{k+1}$ or $\Delta_{k+1} R_{|x|} \Delta_k$ and they can be proved as above. 
\end{itemize}
\medskip
In sum, $\N^\Gamma, \Theta, \sigma^* \vDash \psi$ for any $\Theta$ such that $\Delta R_{|x|}\Theta$. Therefore, $\N^\Gamma, \Delta, \sigma^* \vDash \K_t \psi$.
\end{itemize}
\end{proof}

\medskip
 
It can be easily checked that:  
\begin{lemma}
$\N^\Gamma$ is an epistemic model, i.e., all the $R_{|x|}$ are equivalence relations. 
\end{lemma}
The following is straightforward by using some new variables but leaving infinitely many new variables still unused. 
\begin{lemma}
Each \SELAS-consistent set $\Gamma^{--}$ can be extended to a consistent set $\Gamma^-$ w.r.t.\ some infinitely proper sublanguage $L$ of $\ELAS^+$ such that for each $a\in \Nm$ there is an $x\in \Va(L)$ such that $x\approx a\in \Gamma^-$. Finally we can extend it to an MCS $\Gamma$ w.r.t.\ $L$.
\end{lemma}

\begin{theorem}
$\SELAS$ is sound and strongly complete over epistemic Kripke models with assignments. 
\end{theorem}
\begin{proof}
Soundness is from Theorem~\ref{th.soundness}. Then given a consistent set $\Gamma^-$, using the above proposition we have a $\Gamma$. By the Truth Lemma \ref{lemtruth} we have a model satisfying $\Gamma$ and hence $\Gamma^-$. 
\end{proof}
From the above proof, it is not hard to see that we can obtain the completeness of \SELAS\ without $\AxTrK, \AxTransK, \AxEucK$ over arbitrary models by some minor modifications of the proof.

\section{Discussions and future work}  
\label{sec.con}
In this paper, we proposed a lightweight epistemic language with assignment operators from dynamic logic, which can express various \textit{de re}\slash \textit{de dicto} readings of knowledge statements when the references of the names are not commonly known.  We gave a complete axiomatisation of the logic over epistemic models with constant domain of agents.

The complexity of the epistemic logic \SELAS\ is currently unknown to us though we conjecture it is decidable due to the very limited use of quantifiers. Under the translation in Section \ref{sec.pre}, the name-free fragment can be viewed as a guarded fragment of first-order logic with transtive guards \cite{Szwast2001}, which implies decidability. However, the non-rigid names translate into function symbols in first-order language, which may cause troubles since the guarded fragment with function symbols in general yields an undecidable logic \cite{Gradel98onthe}. We are not that far from the decidability boundary, if not on the wrong side. 

To actually design a tableaux method in pursuing the decidability of our logic, we have to handle the difficulties from various sources: 
\begin{itemize}
\item S5 frame conditions
\item equalities
\item constant domain
\item non-rigid names
\item termed modalities
\item assignment operators
\end{itemize}
Some of the issues are already complicated on their own based on the knowledge of existing work. The biggest hurdle for the termination in a tableau method for S5-based logic like the ones proposed in \cite{Fitting98,Massacci94}, is to ensure loops in finite steps. This requires us in our setting to show that given a satisfiable formula, we can bound the number of necessary elements in the domain (for non-rigid names) and the number of subformulas we may encounter when building the tableau. The S5 condition and the assignment operator may ask us to always introduce new elements in the domain when creating new successors, while the new elements can essentially create new subformulas, if we add new symbols for them in the tableaux.   
 On the other hand, without the transitivity and symmetry conditions, it is possible to bound the number of new elements in the domain to obtain decidability via some finite model property. We leave the details to a future occasion as well as the exploration of other ideas for decidability such as filtering the canonical model. 

\medskip

Below we list a few other further directions:
\begin{itemize}
\item Model theoretical issues of \ELAS.
\item Extension with function symbols.  
\item Extension with a (termed) common knowledge operator. 
\item Extension with limited quantifications over agents as in \cite{naumov2018everyone}. 
\item Extension to varying domain models, where the existence of all the agents is not commonly known. 
\end{itemize}
Finally, as a general direction, it would be interesting to consider what happens if we replace the standard epistemic logic with our $\ELAS$ in various existing logical framework extending the standard one. 
\paragraph{Acknowledgement}The authors thank Johan van Benthem,  Rasmus Rendsvig, and Dominik Klein for pointers on related work. The authors are also grateful to the anonymous reviewers of AiML, whose comments helped in improving the presentation of the paper.\footnote{Including the suggestion to change the previous title `Call me by your name' of the paper.}
The research for this work was supported by the New Zealand Centre at Peking University and Major Program of the National Social Science Foundation of China (NO. 17ZDA026).

\end{document}